\documentclass{article}
\usepackage[T1]{fontenc}
\usepackage{ijcai26}

\usepackage{times}
\usepackage{soul}
\usepackage{url}
\usepackage[hidelinks]{hyperref}
\usepackage[utf8]{inputenc}
\usepackage[small]{caption}
\usepackage{graphicx}
\usepackage{amsmath}
\usepackage{amsthm}
\usepackage{booktabs}
\usepackage{algorithm}
\usepackage{algorithmic}
\usepackage[switch]{lineno}
\usepackage{paralist}
\usepackage{amssymb}
\usepackage{tikz}
\usepackage{amsmath}
\usepackage{listings}
\usepackage{multirow}
\usetikzlibrary{positioning, arrows.meta, calc, quotes, bending}

\floatstyle{ruled}
\newfloat{listing}{tb}{lst}{}
\floatname{listing}{Listing}

\newtheorem{example}{Example}
\newtheorem{theorem}{Theorem}

\title{Automated Approach for Solving Infinite-state Polynomial Reachability Games}

\author{
Krishnendu Chatterjee$^1$
\and
Ehsan Kafshdar Goharshady$^1$\and
Mehrdad Karrabi$^{1}$\and\\
Maximilian Seeliger$^2$\And
\DJ or\dj e Žikelić$^3$
\affiliations
$^1$ Institute of Science and Technology Austria (ISTA)\\
$^2$ ETH Zurich\\
$^3$ Singapore Management University\\
\emails
\{krishnendu.chatterjee, ehsan.goharshady, mehrdad.karrabi\}@ist.ac.at,
mseeliger@ethz.ch,
dzikelic@smu.edu.sg
}

\newcommand{\vars}{\mathcal{V}}
\newcommand{\init}{\mathit{init}}
\newcommand{\states}{\mathcal{S}}
\newcommand{\rstates}{\states_\mathit{Reach}}
\newcommand{\sstates}{\states_\mathit{Safe}}
\newcommand{\R}{\mathbb{R}}
\newcommand{\target}{\mathcal{O}}
\newcommand{\REACH}{\texttt{REACH}}
\newcommand{\SAFE}{\texttt{SAFE}}

\newcommand{\labels}{\mathcal{L}}
\newcommand{\rlabels}{\labels_\mathit{Reach}}
\newcommand{\slabels}{\labels_{\mathit{Safe}}}

\newcommand{\val}{\mathbf{x}}
\newcommand{\valinit}{\val_\init}
\newcommand{\labelinit}{l_\init}

\newcommand{\game}{\mathcal{G}}
\newcommand{\N}{\mathbb{N}}
\renewcommand{\succ}{\mathit{succ}}
\newcommand{\rstrat}{\sigma_{\mathit{Reach}}}
\newcommand{\trans}{\mapsto}
\newcommand{\transition}{\tau}
\newcommand{\connected}{accessible\ }
\newcommand{\update}{U}
\newcommand{\guard}{G}
\newcommand{\strategy}{\sigma}
\newcommand{\pre}{\mathit{pre}}
\newcommand{\valspace}{X}
\renewcommand{\paragraph}[1]{{\smallskip\noindent\textit{\textbf{#1}}}}
\newtheorem*{theorem*}{Theorem}

\newtheorem{definition}[theorem]{Definition}

\begin{document}

\maketitle

\begin{abstract}
	Reachability games are two-player games played on a graph, where the objective of $\REACH$ player is to reach the target set whereas the objective of $\SAFE$ player is to stay away from the target set. Reachability games have important applications in artificial intelligence and reactive synthesis, and many of these applications give rise to infinite-state reachability games. 
	%
	In this paper, we study turn-based reachability games on infinite-state graphs defined over valuations of a finite set of real variables. We consider the problem of determining the existence of and computing a winning strategy for $\REACH$ player. Our contributions are twofold. First, we propose ranking certificates for reachability games, a sound and complete proof rule for proving that $\REACH$ player has a winning strategy from the specified initial state. Second, we consider polynomial reachability games, where transitions and objectives are described by polynomial constraints over real variables, and propose a fully automated algorithm for computing a winning strategy for $\REACH$ player together with a formal correctness witness in the form of a ranking certificate.
	The algorithm is sound, semi-complete, and runs in sub-exponential time. Our experiments demonstrate the ability of our method to solve challenging examples from the literature that were out of the reach of existing methods. Specifically, for the classical Cinderella-Stepmother game, we are able to compute an optimal winning strategy for an arbitrary precision parameter for the first time.

\end{abstract}

%
\section{Introduction} \label{sec:intro}

A standard way of modeling adversarial semantics is considering games played on graphs defined as a turn-based game on a graph where the nodes are partitioned between the two players. Such graph games have been widely studied in the literature and have many applications in formal verification, control synthesis, and artificial intelligence. For instance, alternating Turing machine \cite{DBLP:journals/jacm/ChandraKS81}, decision making in multi-agent systems \cite{DBLP:conf/nips/00050ZZBF22,DBLP:conf/ijcai/0001WSZ21}, planning \cite{ahuja2022using,skrynnik2024learn}, opponent modeling \cite{yu2022model,DBLP:conf/ijcai/0001WSZ21,li2024opponent} and controller synthesis \cite{AsarinMP94,DBLP:conf/focs/AlurHK97} are all modeled via games on graphs. Given a winning condition and a game graph, solving the game corresponds to deciding which player (if any) has a winning strategy, together with finding such a winning strategy.

One of the most fundamental and well-studied problems in graph games considers \emph{reachability} and \emph{safety} objectives, where the goal of the first ($\REACH$) player is to force a specific target set of states to be reached, while the second ($\SAFE$) player's goal is to prevent this from happening. In safety-critical systems, a winning strategy for $\SAFE$ induces a safe controller, and a winning strategy for $\REACH$ corresponds to a buggy behavior in the system. Similarly, in planning, a winning strategy for $\REACH$ induces a controller that reaches the desired target despite the adversarial behavior of the environment.
There are other objectives that have been considered in graph games. For instance, LTL and $\omega$-regular logical specifications \cite{DBLP:conf/focs/AlfaroHK98,DBLP:journals/jacm/ChatterjeeH14} are qualitative objectives, and mean payoff, discounted sum, and total payoff are quantitative objectives~\cite{DBLP:conf/atal/Steeples0W21,DBLP:conf/aaai/VorobeychikS12,DBLP:conf/aaai/SokotaLTDDBSBL21} which are not the focus of this work.

Graph games can be defined on a finite or infinite set of states. It is well known that reachability games, in both cases, have memoryless determinacy, i.e., from each state either $\REACH$ or $\SAFE$ wins with a strategy that depends only on the current state of the game, not the history of the play~\cite{martin1975borel,Mazala01}.
Infinite state games can be represented via many tools, e.g., pushdown automata~\cite{DBLP:conf/concur/BouajjaniEM97}, timed automata~\cite{DBLP:journals/tcs/AlurD94}, or transition systems~\cite{DBLP:journals/pacmpl/HeimD24}.
While determining the winner from each state in finite state games can be done in linear time~\cite{DBLP:journals/corr/abs-0805-1391,DBLP:conf/focs/AlfaroHK98}, Rice's theorem~\cite{rice1953classes} shows that the same task in general infinite state games is undecidable even on games where one of the players has only one possible action in each state. The undecidability result implies that one cannot hope for a sound and complete algorithm for solving reachability games. However, many works have examined restricted classes of games and proposed decision procedures on them, e.g. \cite{FarzanK18} focuses on games definable in linear arithmetic while \cite{DBLP:conf/aaai/FaellaP23} studies games where $\SAFE$ has a bounded number of choices in each turn. In our case, we consider games defined via polynomial expressions over game variables and propose a sound and semi-complete algorithm for solving them.

\paragraph{Contributions.}
In this work, we consider turn-based infinite-state reachability games defined over the set of all valuations of finitely many real-valued variables, also sometimes called arithmetic reachability games~\cite{FarzanK18}. We consider the problem of determining the existence of and computing a winning strategy for $\REACH$ player. Our contributions are as follows:
\begin{compactitem}
    \item {\em Theory.} We introduce {\em ranking certificates for reachability games}, a sound and complete proof rule for proving the existence of a winning strategy of $\REACH$ player. Intuitively, our ranking certificate is a function which assigns a real value to each state of the game that is required to be non-negative and to strictly decrease along a game play until the target set is reached, thus acting as a witness that the game is always making progress towards reaching the target set. Our ranking certificates draw insight from the classical notion of ranking functions for proving termination of programs~\cite{Floyd1967} and generalize it to the setting of reachability games (Section~\ref{sec:proofrule}).

    \item {\em Automation.} We present a fully automated algorithm for computing a winning strategy for $\REACH$ player together with a formal correctness witness in the form of a ranking certificate in {\em polynomial} reachability games, where transitions and the target set are described by polynomial constraints over real variables. We show that our algorithm is sound, semi-complete, and runs in sub-exponential time (Section~\ref{sec:algo}).

    \item {\em Experiments.} We implement our algorithm and perform case studies to show its practical applicability. We consider the classical Cinderella-Stepmother game~\cite{DBLP:conf/ifipTCS/BodlaenderHKSWZ12}, for which our algorithm is able to compute a winning strategy for $\REACH$ player for an arbitrary precision parameter for the first time (Section~\ref{sec:experiments}).
\end{compactitem}

\paragraph{Novelty and limitations.} Our contributions significantly advance the study of infinite-state reachability games.
\begin{compactitem}
    \item \emph{Theoretical improvements.}  The previous sound and semi-complete approach with complexity guarantees is limited to linear arithmetic and has double-exponential complexity \cite{FarzanK18}. In contrast, our approach handles polynomial arithmetic and has sub-exponential complexity, i.e., we solve a wider range of games with lower complexity.
    \item \emph{Solving classical games.} For the Cinderella-Stepmother game (Example~\ref{example:cinderella}), while it is known that for bucket capacity less than 2 $\REACH$ has a winning strategy, the method of \cite{DBLP:conf/popl/BeyeneCPR14} could solve the game only up to capacity 1.4. This was later improved by the semi-complete method of \cite{FarzanK18} for capacity up to $1.7$. More recently, \cite{DBLP:conf/sigsoft/SamuelDK21} proposed a method that handles capacity up to $2-10^{-20}$ but provides no completeness guarantees. Our approach is the first to solve this game for bucket capacity $2-\epsilon$, where $\epsilon>0$ is a variable, i.e., we can solve the game for arbitrary precision parameter.
\end{compactitem}

Our approach has two main limitations: (i) our algorithm applies only to polynomial games, and extension to more general classes of games is an interesting direction, and (ii) it focuses primarily on reachability games, and extension to other objectives is another interesting direction of research.

\paragraph{Related work.}
Infinite-state reachability games have received increased interest in recent years. One line of work reduces this problem to logical satisfiability problems and then uses off-the-shelf solvers to check satisfiability. For instance, ~\cite{DBLP:conf/popl/BeyeneCPR14} reduce the problem to constrained Horn clause (CHC) solving and~\cite{DBLP:conf/aaai/FaellaP23} use a similar translation for a class of games in which the number of choices of $\SAFE$ player is bounded. \cite{FarzanK18} consider linear reachability games and propose a method for solving them by gradually increasing the finite-time horizon of the game, until a winning strategy for one of the players is computed. Another line of work is based on using fixed-point computation-based methods to compute winning strategies~\cite{DBLP:journals/pacmpl/HeimD24,DBLP:conf/cav/SchmuckHDN24,DBLP:journals/pacmpl/UnnoTGK23}. These methods do not provide completeness guarantees and the fixed-point computation may not terminate. \cite{DBLP:conf/sigsoft/SamuelDK21} develop a fixed-point computation-based method and show its scalability by outperforming other existing methods on challenging benchmarks. \cite{BaierCFFJS21} use Craig interpolation to achieve sufficient sub-goals and necessary sub-goals for $\REACH$ player to divide the game into simpler sub-games, which are easier to solve. In comparison to all these works, our automated method is the first to provide the following desirable features: (i) it is applicable to polynomial reachability games (unlike~\cite{FarzanK18}), (ii) it provides semi-completeness guarantees (unlike~\cite{DBLP:journals/pacmpl/HeimD24}), and (iii) it runs in sub-exponential time.


\section{Preliminaries}\label{sec:prelims}

In this section, we define the model of infinite-state reachability games that we consider in this work. We use $\R$ and $\N$ for the sets of real numbers and of non-negative integers. 

\begin{definition}[Reachability game]
    A two-player (infinite-state) reachability game between players $\REACH$ and $\SAFE$ is a tuple $\game = (\vars,\valspace, \valinit, \rlabels, \slabels, \labelinit, \trans)$ where:
    \begin{compactitem}
        \item $\vars$ is a finite set of real-valued {\em variables};
        \item $\valspace \subseteq \R^{|\vars|}$ is the set of allowed valuations of variables; 
        \item $\valinit \in \valspace$ is the {\em initial variable valuation};
        \item $\rlabels$ and  $\slabels$ are two disjoint sets of labels. We use $\labels = \rlabels \cup \slabels$ for the set of all {\em labels} in the game; 
        \item $\labelinit \in \labels$ is the {\em initial label} of the game;
        \item $\trans$ is a finite set of {\em transitions}. Each $\transition \in \, \trans$ is a tuple $\transition = (l, l', \guard_\transition, \update_\transition)$ where $l$ and $l'$ are the source and the target labels, $\guard_\transition \subseteq \valspace$ is the guard of the transition that specifies for which variable valuations the transition can be taken, and $\update_\transition$ is a relation in $\guard_\transition \times \valspace$ that specifies the set of all possible (not necessarily unique) variable valuation updates upon taking the transition.
    \end{compactitem}
\end{definition}

A {\em state} in the game is a tuple $(l,\val)$ consisting of a label $l \in \labels$ and a variable valuation $\val \in \valspace$. The state $(\labelinit,\valinit)$ is said to be the {\em initial state} of the game. For two states $s_1=(l_1,\val_1)$ and $s_2=(l_2,\val_2)$, we say that $s_2$ is a {\em successor} of $s_1$ and write $s_1 \trans s_2$ if there exists a transition $\tau$ with source label $l_1$ and target label $l_2$ such that $\val_1 \in \guard_\tau$ and $(\val_1,\val_2) \in \update_\tau$. For each state $s \in \states$, we use $\succ(s) := \{s'| s \trans s'\}$ to denote the set of successors of $s$. We write $\rstates = \rlabels \times \valspace$ and $\sstates = \slabels \times \valspace$ and say that these states {\em belong to} players $\REACH$ and $\SAFE$, and write $\states = \rstates \cup \sstates$ for the set of all states.

We assume that for each label $l \in \labels$ we have $\bigcup_{(l, l', \guard_\transition, \update_\transition) \in \trans} \guard_\transition = \valspace$, meaning that it is always possible to take at least one transition and that there are no deadlock states in the game. This assumption is imposed without loss of generality, as we may add two auxiliary labels belonging to the opposite player with self-loop transitions, to which the game transitions whenever a deadlock state is reached. Moreover, we assume that $\guard_\tau \cap \guard_{\tau'} = \emptyset$ for any two transitions $\tau,\tau'$ outgoing from the same label $l$, which is also imposed without loss of generality as one can merge the update relations of two transitions at states in which the guards of both transitions are satisfied. These two assumptions together ensure that, for every state $s = (l,\val)$, there is a unique transition $(l,l',\guard_\tau,\update_\tau)$ with $\val \in \guard_\tau$. 

\paragraph{Semantics.} A reachability game is played between two players $\REACH$ and $\SAFE$, which indefinitely move a token along the states of the game. The token is initially placed at the initial state $(\labelinit,\valinit)$. Then, in each turn, if the token is in a state $s = (l,\val)$, the player to whom the state $s$ belongs gets to move the token. That player takes the unique transition $\tau = (l,l',G_\tau,\update_\tau)$ with $\val \in \guard_\tau$ and chooses a variable valuation $\val'$ such that $(\val,\val') \in \update_\tau$. The token is then moved to the successor state $(l',\val')$. This is repeated indefinitely and gives rise to an infinite sequence $\rho = (s_0, s_1, \dots)$ where $s_0 = s_\init$ and $s_i \trans s_{i+1}$ for all $i \in \N$, which we call a {\em play}. A {\em finite play} is a finite prefix of a play. 

\paragraph{Strategies.} A strategy of a player $P \in \{\REACH,\SAFE\}$ is a function $\strategy_P\colon \states^* \times \states_P \to \states$ which maps every finite play ending in a state that belongs to the player $P$ to a successor state of the last state. A strategy $\sigma_P$ is said to be {\em memoryless}, if for every two finite plays $h$ and $h'$ that end in the same state belonging to player $P$, we have $\sigma_P(h) = \sigma_P(h')$. Every two strategies $\sigma_\REACH$ and $\sigma_\SAFE$ of the two players induce a unique play which we denote by $\rho(\sigma_\REACH,\sigma_\SAFE)$.

\paragraph{Winning condition: Reachability.} In this work, we study infinite-state reachability games. A reachability game is equipped with a set of {\em target (or objective) states} $\target \subseteq \states$. A play $\rho$ is said to be {\em winning for $\REACH$} if it contains a state in the target set, otherwise, it is {\em winning for $\SAFE$}.

A strategy $\sigma_P$ of a player $P \in \{\REACH,\SAFE\}$ is said to be {\em winning} if, for every strategy $\sigma$ of the other player, the induced play is winning for the player $P$. It is a classical result in turn-based games that, in reachability games, there exists a winning strategy for $P \in \{\REACH,\SAFE\}$ if and only if there exists a memoryless winning strategy for $P$~\cite{Mazala01}. Hence, we, without loss of generality, restrict our attention to memoryless strategies.

\paragraph{Problem statement.} Given a reachability game $\game$ with target set $\target$, the goal is to decide whether $\REACH$ has a winning strategy and, if yes, compute such winning strategy.

\begin{example} \label{example:cinderella}
    Fig. \ref{fig:cinderella} shows the Cinderella-Stepmother game, which is a classical example of an infinite-state reachability game~\cite{DBLP:conf/ifipTCS/BodlaenderHKSWZ12}. The game is defined over five real-valued variables ($b_0,\dots, b_4$) which denote capacities of five buckets that are placed along a circle. The buckets are initially empty, hence the initial variable valuation is $(0,\dots,0)$. Each bucket has maximum capacity $U$, hence the value of the variables can range over $[0,U+1]^5$. There are two labels $\texttt{C}$ and $\texttt{SM}$ that belong to the Cinderella and the Stepmother players, respectively, with $\texttt{SM}$ being the initial label. As can be seen from Fig.~\ref{fig:cinderella}, the two players play alternatively, starting with the Stepmother. In each turn, Stepmother distributes 1 liter of water among the five buckets (the transition from \texttt{SM} to \texttt{C}). Then, Cinderella chooses two neighbouring buckets along the circle and empties them (the transition from \texttt{C} to \texttt{SM}). Stepmother's objective is to make one of the buckets overflow; hence, she is the $\REACH$ player with the target set of states $\target = \{(\texttt{C},\mathbf{b})| \bigvee_{0 \leq i \leq 4} \mathbf{b}_i > U\}$. On the other hand, Cinderella's objective is to prevent any bucket from overflowing, hence she is the $\SAFE$ player.
    \begin{figure}[t]
    \centering 
    \begin{tikzpicture}[
    scale=0.7,
    transform shape,
    node distance=6cm, 
    reach node/.style={circle, draw, thick, align=center, minimum width=1cm, minimum height=1cm},
    safe node/.style={rectangle, draw, thick, align=center, minimum width=1cm, minimum height=1cm},
    double node/.style={circle, draw, double, thick, align=center, minimum width=1cm, minimum height=0.6cm},
    every edge/.append style={->, thick, >=Stealth}, 
    math label/.style={font=\large, sloped}
]

    \node[reach node] (stepmother) {\texttt{SM}};
    \node[safe node] (cinderella) [right=of stepmother] {\texttt{C}};

    \draw[<-, thick, >=Stealth] (stepmother.west) -- ++(-1.5cm, 0) node[midway, above] {$\mathbf{b}=0$};
    \path (stepmother) edge[bend left=30,
        "$\bigwedge_{i=1}^5 b_i\leq b_i' \land \sum_{i=1}^5 (b_i'-b_i) = 1$",
        font=\large,sloped, align=center, text width=6cm,
        pos=0.5
    ] (cinderella);

    \path (cinderella) edge[bend left=30] 
  node[font=\large,midway, below, math label] 
  {
    $b'_0 = b_1' = 0 
      \vee  b_1' = b_2' = 0 
      \dots 
      \vee  b_4' = b_0' = 0$
  } 
  (stepmother);

    \end{tikzpicture}
    \caption{Cinderella-Stepmother game defined over $\valspace = [0,U+1]^5$. $\REACH$ owns the label $\texttt{SM}$ with the objective of reaching a state where $b_i > U$ for some $i$.  
    }
    \label{fig:cinderella}
\end{figure}
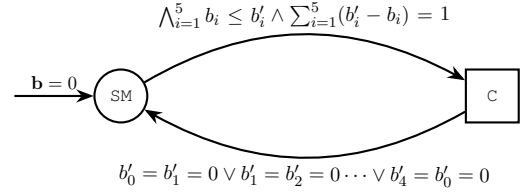
\end{example}

\section{Ranking Certificate for Reachability Games} \label{sec:proofrule}

We now introduce our ranking certificates for reachability games. The results presented in this section are applicable to the {\em general class of reachability games}, and not just polynomial reachability games. In what follows, suppose that $\game = (\vars,\valspace, \valinit, \rlabels, \slabels, \labelinit, \trans)$ is a reachability game with a set of target states $\target \subseteq \states$.

\paragraph{Motivation: Ranking functions for programs.} Our certificate is motivated by the notion of ranking functions in the program verification literature~\cite{Floyd1967,ColonS01,PodelskiR04}, presenting a classical approach to proving termination in programs. A ranking function maps program states to real values and is required to (1) be non-negative at all reachable states and (2) make progress towards termination by decreasing by at least $1$ at every step until a terminal state is reached. If such a function exists, then every run of that program terminates as a ranking function cannot remain non-negative while indefinitely decreasing by at least $1$ at every step. Hence, ranking functions provide a sound (and complete) proof rule for proving termination of programs~\cite{Floyd1967}. 

\paragraph{Our certificate for reachability games.} We now introduce our ranking certificates for reachability games, which provide a formal proof of the existence of a winning strategy of the $\REACH$ player. Similarly to ranking functions in programs, our ranking certificate for reachability games is a function $f: \states \rightarrow \R$ which maps each game state to a real value. However, unlike ranking functions that need to be non-negative and to make progress towards termination at all reachable program states, we encounter two challenges in generalizing this idea to reachability games:
\begin{compactenum}
    \item {\em Not all states are winning for $\REACH$.} In reachability games, there are states from which $\REACH$ does not have a winning strategy. Hence, imposing a progress condition on such states may make our ranking certificates overly conservative. To overcome this challenge, we also use our ranking certificates as indicators of whether a state is winning for $\REACH$ by using $f(s) \geq 0$ as an indicator that $\REACH$ has a winning strategy from a state $s$. In particular, in (C1) in Definition~\ref{def:reach-rank}, we only require $f$ to be non-negative at the initial game state. Then, in (C2) and (C3) in Definition~\ref{def:reach-rank}, which are progress conditions, we only impose the progress condition when $f(s) \geq 0$.
    \item {\em States may belong to different players.} In reachability games, states may belong to either $\REACH$ or $\SAFE$ player, who gets to choose the successor state to which the game token should be moved. This leads to different requirements on the progress condition. In particular, in (C2) in Definition~\ref{def:reach-rank}, we impose the progress condition of $f$ for {\em all} successor states that can be chosen by $\SAFE$. In contrast, in (C3) in Definition~\ref{def:reach-rank}, we only impose the progress condition of $f$ for {\em some} successor state that can be chosen by $\REACH$. That way $\REACH$ can always ensure to make progress towards the target set, while $\SAFE$ cannot prevent progress towards the target set.
\end{compactenum} 
The following definition formalizes the intuition discussed above and formalizes our novel notion of reachability certificates for reachability games.




\begin{definition}[Ranking certificates for reachability games]\label{def:reach-rank}
Given a reachability game $\game = (\vars,\valspace, \valinit, \rlabels, \slabels, \labelinit, \trans)$ with a set of target states $\target \subseteq \states$, a {\em ranking certificate} is a function $f\colon \states \to \R$ that satisfies the following conditions:
\begin{compactenum}
    \item[(C1)] $f(s_\init) \geq 0$, i.e. $f$ is non-negative at the initial state.
    \item[(C2)] If $s \in \sstates \setminus \target$ and $f(s) \geq 0$, then 
    \[ \forall s' \in \succ(s).\, f(s)-1 \geq f(s') \geq 0. \]
    \item[(C3)] If $s \in \rstates \setminus \target$ and $f(s) \geq 0$, then 
    \[ \exists s' \in \succ(s).\, f(s)-1 \geq f(s') \geq 0. \]
\end{compactenum}
    In condition (C3) above, we refer to every $s' \in \succ(s)$ for which the condition is satisfied as a {\em ranking successor} of $s$ with respect to $f$.
\end{definition}


\paragraph{Soundness.} The following theorem, proved in the Appendix, establishes soundness of our ranking certificates. In particular, it shows that if there exists a ranking certificate in a reachability game, then there exists a winning strategy for $\REACH$ player. The theorem also shows how a winning strategy can be constructed from the ranking certificate.

\begin{theorem}[Soundness]\label{thm:soundness}
    Suppose that there exists a ranking certificate $f$ in a reachability game $\game$ with a set of target states $\target \subseteq \states$. Then, there exists a memoryless winning strategy for $\REACH$ player. 
    
    An example of such a strategy is a strategy $\rstrat^f$ that maps each state $s \in \rstates$ to a ranking successor of $s$ if $f(s) \geq 0$, and to an arbitrary successor of $s$ if $f(s) < 0$.
\end{theorem}

\paragraph{Completeness.} Although theorem~\ref{thm:soundness} shows the soundness of our ranking certificates, it turns out that the proof rule is not complete. The following is an example of a reachability game in which the $\REACH$ player has a memoryless winning strategy, but there does not exist a ranking certificate to witness this. 

\begin{example} \label{example:incomplete}
    
Consider the reachability game in Figure \ref{fig:incomplete}. In this game, starting from the initial label $\texttt{S}$ that belongs to $\SAFE$ and the variable valuation $x=0$, $\SAFE$ adversarially assigns an arbitrary real value to $x$. The game then moves to the label $\texttt{R}$ that belongs to $\REACH$, and in every step $\REACH$ decrements the value of $x$ by $1$. The set of target states in this reachability game is $\target = \{(\texttt{R},x)|x<0\}$.  

It is clear that $\REACH$ has a memoryless winning strategy by always decrementing $x$ by $1$. That is because, regardless of the value of $x$ that $\SAFE$ chooses, the value of $x$ will eventually become negative. However, this reachability game turns out not to admit a ranking certificate. To see this, suppose, for the sake of contradiction, that there exists a ranking certificate $f$. Then $f$ must satisfy the following, by conditions (C1)-(C3) in Definition~\ref{def:reach-rank}:
\[
\begin{split}
    (C1): \,\,\,&f(\texttt{S},0)\geq 0 \\
    (C2): \,\,\,&\forall x \in \R.\, f(\texttt{S},0) -1 \geq f(\texttt{R},x) \geq 0 \\
    (C3): \,\,\,&\forall x \in \R.\, x \geq 0 \land f(\texttt{R},x) \geq 0 \\
    &\Longrightarrow f(\texttt{R},x) -1 \geq f(\texttt{R},x-1) \geq 0 
\end{split}
\]
The second condition above implies that $f(\texttt{R},x) \geq 0$ for all $x \in \R$. Thus, it follows from the third condition above and by a simple induction that $f(\texttt{R},n) \geq n$ for $n \in \N$. Then by (C2), it is directly implied that $f(\texttt{S},0) \geq n$ for all $n \in \N$ which means $f(\texttt{S},0)$ cannot have any finite value and the reachability game does not admit a ranking certificate.

\begin{figure}[t] 
    \centering 
    \begin{tikzpicture}[
    scale=0.7,
    transform shape,
    node distance=5.5cm, 
    reach node/.style={circle, draw, thick, align=center, minimum width=1cm, minimum height=1cm},
    safe node/.style={rectangle, draw, thick, align=center, minimum width=1cm, minimum height=1cm},
    every edge/.append style={->, thick, >=Stealth}, 
    math label/.style={font=\small, sloped} 
]

    \node[safe node] (left) {$\texttt{S}$};
    \node[reach node] (middle) [right of=left] {$\texttt{R}$};

    \path (left) edge["$x' = ?$", math label] (middle);


    \path (middle) edge[loop below, in=250, out=290, looseness=5, "$x' = x - 1$" below, math label] (middle);

    \draw[<-, thick, >=Stealth] (left.west) -- ++(-1.5cm, 0) node[midway, above] {$x=0$};
    
    \end{tikzpicture}
    \caption{A reachability game demonstrating incompleteness of our ranking certificates.} 
    \label{fig:incomplete} 
\end{figure}
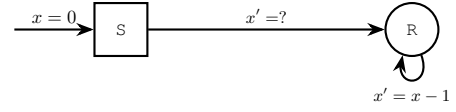

\end{example}


Intuitively, the incompleteness issue in Example~\ref{example:incomplete} arises since $\SAFE$ can choose between infinitely many choices for the value of $x$. This allows $\SAFE$ to enforce an arbitrarily large lower bound on the value of $f$ at the initial state of the game, and no finite value would be sufficient. It also hints that, if we were to restrict $\SAFE$ by only providing it with finitely many choices in each step of the game, it may be possible to overcome the incompleteness issue. The following theorem shows that this is indeed the case, and that ranking certificates for reachability games are both sound and complete if $\SAFE$ is restricted to having only finitely many choices at each step. The proof is provided in the Appendix.

\begin{theorem}[Completeness for finite $\SAFE$ choices]\label{thm:completeness}
    Let $\game$ be a reachability game with a set of target states $\target \subseteq \states$. Suppose that $\SAFE$ has only finitely many choices at each step, i.e. for every state $s \in \sstates$ the set of successor states $\succ(s)$ is finite. If $\REACH$ has a winning strategy, then there exists a ranking certificate for this reachability game.
\end{theorem}

\paragraph{Novelty of certificates.} Unlike classical attractor computation (progress measure) \cite{Mazala01} methods that compute the minimal alternating distance, our notion of ranking certificates only requires a strict decrease in its value. Relaxing this optimality constraint allows for much simpler witnesses, i.e. in the form of polynomial functions.


\section{Algorithm for Solving Reachability Games} \label{sec:algo}

We now present our automated algorithm for proving the existence of and computing a winning strategy of $\REACH$ player. Given that this problem is undecidable even in 1-player games due to Rice's theorem, we cannot hope for an algorithm that is both sound and complete. Instead, in what follows we design a {\em sound and semi-complete} algorithm. By soundness, we mean that the output of our algorithm and the produced winning strategy are guaranteed to be correct. By semi-completeness, we mean that our algorithm is guaranteed to compute a winning strategy for $\REACH$ player whenever it exists, for a subclass of reachability games that we formally characterize below.


\paragraph{Assumption: Polynomial reachability games.} To enable full automation, we restrict our algorithm to polynomial reachability games. A reachability game $\game$ with target states $\target \subseteq \states$ is said to be {\em polynomial}, if 
    (i) For every transition $\tau \in\, \mapsto$, the guard $\guard_\tau \subseteq \mathbb{R}^{\vars}$ is a set of boolean combination of polynomial inequalities over variables~$\vars$;
    (ii) For every transition $\tau \in\, \mapsto$, the update $\update_\tau \subseteq \guard_\tau \times \valspace$ is a set of boolean combination of polynomial inequalities over variables $\vars \cup \vars'$, where $\vars'$ denote the variables $\vars$ upon update;
    (iii) For each label $l \in \labels$, the set of target variable valuations $\{\val \mid (l,\val) \in \target\}$ can be represented as a set of a boolean combination of polynomial inequalities over variables $\vars$;
    (iv) The set of allowed variable valuations $\valspace \subseteq \R^{|\vars|}$ is a satisfiability set of a boolean combination of polynomial inequalities over variables $\vars$.


\paragraph{Algorithm overview.} Our algorithm proves the existence of a memoryless winning strategy for $\REACH$ by synthesizing a ranking certificate $f:\states \to \R$ and a memoryless winning strategy $\strategy:\rstates \to \states$ for $\REACH$. These two objects are searched for and computed simultaneously by following a template-based synthesis approach. The algorithm proceeds in four steps. In Step~1, the algorithm fixes symbolic polynomial template expressions that define the ranking certificate $f$ and the strategy $\sigma$, one for each label in the game. In Step~2, the algorithm collects a system of constraints over the symbolic template variables that together encode that $f$ defines a valid ranking certificate and that $\sigma$ is a winning strategy induced by $f$ as in Theorem~\ref{thm:soundness}. In Step~3, these constraints are simplified in order to reduce them to a sentence in the purely existentially quantified theory of reals. Finally, in Step~4, the constraints are solved by an off-the-shelf SMT solver, and any solution gives rise to a concrete instance of a valid ranking certificate and a winning strategy. In what follows, we present the details behind each of the four steps.

\paragraph{Step 1: Setting up polynomial templates.} The algorithm takes as input a polynomial degree parameter $D \in \N$, and it fixes symbolic polynomial templates of maximal polynomial degree $D$ for expressions that define the ranking certificate $f$ and the strategy $\sigma$. Denote by $M_\vars^D = \{m_0, \dots, m_k\}$ the set of all monomials of degree at most $D$ over the set of variables $\vars$. The template for the strategy $\sigma$ at each game label $l \in \labels$ and variable $v \in \vars$ is defined via
\begin{align}\label{formula:poly-strat}
\strategy^v_l(\val) = \sum_{i=0}^k t^{l,v}_i\cdot m_i
\end{align}
where each $t^{l,v}_i$ is a real-valued symbolic template variable, and $\strategy^v_l(\val)$ evaluates to the value of variable $v$ upon a one-step execution of the strategy $\sigma$ at state $(l,\val)$. The template for the ranking certificate $f$ at each game label $l \in \labels$ is defined via
\begin{align} \label{formula:poly-rank}
f_l(\val) = \sum_{i=0}^k s^{l}_i\cdot m_i
\end{align}
where each $s^l_i$ is a real-valued symbolic template variable, and $f_l(\val)$ evaluates to the value of the ranking certificate $f$ at state $(l,\val)$. The values of symbolic template variables $t^{l,v}_i$ and $s^l_i$ are at this point unspecified, and the goal of template-based synthesis is to compute concrete values of template variables that together give rise to a valid ranking certificate $f$ and a winning strategy $\sigma$.


\begin{example}
    Consider the Cinderella-Stepmother game presented in Example~\ref{example:cinderella}, and let $D = 1$ be the polynomial degree parameter. The variable set in this example is $\vars = \{b_0,b_1,b_2,b_3,b_4\}$ and set of all monomials of degree at most $1$ over $\vars$ is $M_\vars^1 = \{1, b_0, b_1, b_2, b_3, b_4\}$. So, in Step~1, the following template is set for $\strategy^{b_0}_{\texttt{R}}$ which evaluates to the value of variable $b_0$ upon executing the strategy at state $(\texttt{R},\mathbf{b})$:
    \[
    \strategy^{b_0}_{\texttt{R}}(\mathbf{b}) = t_0^{\texttt{R},b_0} + t_1^{\texttt{R},b_0} \cdot b_0 + t_2^{\texttt{R},b_0}\cdot b_1 + t_3^{\texttt{R},b_0}\cdot b_2 + t_4^{\texttt{R},b_0}\cdot b_3 + t_5^{\texttt{R},b_0}\cdot b_4
    \]
    Similarly, the following template is fixed for $f_{\texttt{R}}$ which evaluates to the value of $f$ at state $(\texttt{R},\mathbf{b})$:
     \[f_\texttt{R}(\mathbf{b}) =  s_0^{\texttt{R}} + s_1^{\texttt{R}} \cdot b_0 + s_2^{\texttt{R}}\cdot b_1 + s_3^{\texttt{R}}\cdot b_2 + s_4^{\texttt{R}}\cdot b_3 + s_5^{\texttt{R}}\cdot b_4 \]
\end{example}

\paragraph{Step 2: Collecting defining constraints.} The algorithm now collects a system of constraints over the symbolic template variables introduced in the previous step, which together encode that $f$ is a valid ranking certificate and that $\sigma$ is a strategy induced by $f$ as in Theorem~\ref{thm:soundness}. In each constraint, every appearance of $f_l$ and $\sigma^v_l$ is substituted by the symbolic polynomial template expression introduced in Step~1:
\begin{itemize}
    \item[(C1)] For condition (C1) in Definition~\ref{def:reach-rank}, the constraint $f_{l_\init}(\boldsymbol{x}_\init) \geq 0$ is collected.
    \item[(C2)] For condition (C2) in Definition~\ref{def:reach-rank}, for each label $l \in \slabels$ and each transition $\tau = (l, l', \guard_\transition, \update_\transition) \in \,\trans$, the algorithm collects the following constraint
    \[
    \begin{split}
    &\forall \val,\val' \in \valspace.\, \\
    &\bigg((\val \vDash \neg \target_l \wedge \guard_\transition) \wedge (f_l(\val) \geq 0) \wedge (\val,\val') \vDash \update_\transition\bigg) \\
    &\Longrightarrow f_l(\val) -1 \geq f_{l'}(\val') \geq 0,
    \end{split}
    \]
    where $\target_l = \{\val \in \R^{\vars} \mid (l,\val) \in \target\}$, $\guard_\tau$ and $\update_\tau$ are all represented in terms of a boolean combination of polynomial constraints over the variables in $\vars \cup \vars'$, as specified in the algorithm assumption stated above. 
    %
    \item[(C3)] For condition (C3) in Definition~\ref{def:reach-rank} and to enforce that for each state $s$ belonging to $\REACH$ the state $\strategy(s)$ is a ranking successor of $s$ as defined in Theorem~\ref{thm:soundness}, the algorithm collects the following constraint for each label $l \in \rlabels$ and each transition $(l, l', \guard_\transition, \update_\transition) \in \,\trans$
    \[
    \begin{split}
        \forall \val\in \valspace.\, &\bigg((\val \vDash \neg \target_l \wedge \guard_\transition) \wedge (f_l(\val) \geq 0)] \bigg) \\
    \Longrightarrow & (\val,\strategy_l(\val)) \vDash \update_\transition \wedge f_l(\val) -1 \geq f_{l'}(\strategy_l(\val)) \geq 0, 
    \end{split}
    \]
    where $\target_l = \{\val \in \R^{\vars} \mid (l,\val) \in \target\}$ and $\guard_\tau$ are represented in terms of a boolean combination of polynomial constraints over the variables in $\vars$. Intuitively, the right-hand-side of the implication encodes that $\sigma_l(\val)$ is a ranking successor of $(l,\val)$ as defined in Theorem~\ref{thm:soundness}.
\end{itemize}
    \begin{example}
        Consider again the Cinderella-Stepmother game, where the maximum capacity of each bucket is set to $U = 1.8$. For (C1), the algorithm collects $f_{\texttt{SM}}(0) \geq 0$. For (C2), the following constraint is collected for the transition from label $\texttt{C}$ to label \texttt{SM}, as Cinderella is the $\SAFE$ player: 
        \[
        \begin{split}
        \forall \mathbf{b},\mathbf{b'} &\in [0,2.8]^5.\, \bigg(\bigwedge_{0 \leq i \leq 4} b_i \leq 1.8   \wedge f_{\texttt{C}}(\mathbf{b}) \geq 0 \\
        &\wedge \bigvee_{0 \leq i \leq 4} [ \bigwedge_{\substack{0 \leq j \leq 4 \\ j\neq i \\ j \neq (i+1)\%5}} b_j' = b_j \wedge b_i' = b_{(i+1)\%5}' = 0 ] \bigg) \\
        &\Longrightarrow f_{\texttt{C}}(\mathbf{b}) - 1 \geq f_{\texttt{SM}}(\mathbf{b}') \geq 0.
        \end{split}
        \]
        Finally, for (C3), a similar constraint is collected for the transition from $\texttt{SM}$ to \texttt{C}, as Stepmother is the $\REACH$ player. 
    \end{example}
    
    \paragraph{Step 3: Removing $\forall$ quantifiers.} The collected constraints for conditions (C2) and (C3) are of the form 
    \begin{equation}\label{eq:implication}
        \forall \val \in \R^\vars.\, \big(\valspace(\val) \wedge P(\val)\big) \Longrightarrow Q(\val),
    \end{equation}
    with $P(\val)$, $Q(\val)$ and $\valspace$ being boolean combinations of polynomial inequalities over variables $\val$. The presence of universal quantification makes solving these constraints computationally expensive. To make our algorithm more efficient, we apply the translation of~\cite{asadi}, which uses Putinar's theorem (for general polynomials) or Farkas' lemma (for degree $1$ polynomials) to translate constraints of the form as in eq.~\eqref{eq:implication} into a system of purely existentially quantified polynomial constraints. As shown in~\cite{asadi}, the translation is sound in the sense that any satisfying assignment of the resulting system of constraints yields a satisfying assignment of the original system of constraints. Moreover, the translation is complete and guarantees equisatisfiability whenever the satisfiability set of the left-hand-side is compact. For the interest of space, we omit the details of this translation and refer the reader to~\cite{asadi}.

    
    
    \paragraph{Step 4: Solving the constraints.} Step~3 results in a system of purely existentially quantified polynomial constraints over real-valued variables. In the last step, our algorithm calls an off-the-shelf SMT solver to solve this system of constraints. If a solution is found, the computed valuation of symbolic template variables gives rise to a concrete instance of the ranking certificate $f$ and the memoryless winning strategy~$\sigma$. Otherwise, the algorithm returns \texttt{UNKNOWN}.




\begin{theorem}\label{thm:algo}
    Consider a polynomial reachability game $\game = (\vars, \valspace, \valinit, \rlabels, \slabels, \labelinit, \trans)$ with target states $\target$: 
    \begin{compactenum}
        \item {\em (Soundness)} If the algorithm returns a solution $(f,\strategy)$, then $f$ is a valid ranking certificate for this reachability game and $\strategy$ is a memoryless winning strategy of $\REACH$.
        \item {\em (Semi-completeness)} Whenever the set of possible valuations $\valspace$ is compact, if there exist a polynomial ranking certificate $f$ and a polynomial memoryless winning strategy $\strategy$ for $\game$ then, for a sufficiently high value of the maximal polynomial degree parameter $D$, the algorithm is guaranteed to compute them. 
        \item {\em (Complexity)} The runtime of the algorithm is sub-exponential in the size of the input game $\game$ and $\target$, for each fixed value of the maximal polynomial degree $D$. 
    \end{compactenum}
\end{theorem}
 The proof is provided in the Appendix.

\section{Case Studies and Experimental Results} \label{sec:experiments}

\paragraph{Implementation.} We implemented a prototype of our method, which takes as input a polynomial reachability game and tries to compute a ranking certificate and a winning strategy for $\REACH$. For the last two steps of our algorithm, we used PolyQEnt \cite{chatterjee2024polyqent} which automates translation of quantified formulas as in eq.~\eqref{eq:implication} into a system of existentially quantified polynomial constraints. Our prototype then uses Z3 \cite{MouraB08} and MathSAT5 \cite{mathsat5} as backend SMT solvers. All our experiments were done on a 11th Gen Intel i5 machine with 16 GB memory and a timeout of 5 minutes. 

\paragraph{Heuristic.} We implement the following simple heuristic, which helps improve the efficiency of our method. The heuristic applies the $\pre: \mathcal{P}(S) \rightarrow \mathcal{P}(S)$ operator once to expand the target set $\target$. For a set of states $A \subseteq \states$, the set $\pre(A)$ is the set of states that are either (i) $\REACH$ states that have at least one successor in $A$, or (ii) $\SAFE$ states whose all successors are in $A$. This is the classical one-step attractor operator, which is common in the analysis of games on graphs. We apply the $\pre$ operator once to replace $\target$ by $\pre(\target)$. Note that $\REACH$ player has a 1-step winning strategy from $\pre(\target)$, hence this heuristic is sound.

We employ this heuristic because we observed that, in many reachability games in the literature, the last step of the game is qualitatively different from earlier steps. For instance, in the last step of the Cinderella-Stepmother game in Example~\ref{example:cinderella}, the stepmother will pour 1 liter of water into the bucket that will overflow. We observed that considering $\pre(\target)$ rather than $\target$ improves the efficiency of our tool. 

\paragraph{Baselines.} We consider two state-of-the-art infinite-state reachability game solvers as baselines: (i) \texttt{GenSys} \cite{DBLP:conf/sigsoft/SamuelDK21}, which uses a classical fixed-point algorithm generalized to the infinite-state setting, and (ii) \texttt{rpgsolve} \cite{DBLP:journals/pacmpl/HeimD24}, which introduces acceleration into fixed-point computation to avoid divergence. 
We choose these since they are recent methods that outperformed previous methods in their experiments.

\paragraph{Case study: Cinderella-Stepmother game.} We experimentally evaluate our method on two variants of the Cinderella-Stepmother game. The first is the classical variant of the game~\cite{DBLP:conf/ifipTCS/BodlaenderHKSWZ12}, which was presented in Example~\ref{example:cinderella}. The second is a non-linear variant of the game in which the stepmother is allowed to distribute~$1$ liter of water across five buckets in any way that ensures $\|\mathbf{b}-\mathbf{b'}\|_2 \leq 1$, rather than $\|\mathbf{b}-\mathbf{b'}\|_1 \leq 1$ as in Example~\ref{example:cinderella}. This introduces non-linearity in the transition relation from $\texttt{SM}$ to $\texttt{C}$, which can only be expressed in polynomial arithmetic. 



In both variants of the game, the stepmother has a winning strategy for any bucket capacity $U < 2$. For example, for the first (classical) variant and for any fixed bucket capacity $U<2$, the following define a valid ranking certificate and a $\REACH$ winning strategy: 
\[
\begin{split}
f(l,\mathbf{b}) &= \begin{cases}
    \frac{U - \max(b_1,b_3)}{2-U} \times 2 & l=\texttt{SM} \\
    \frac{U - \max(b_1,b_3)}{2-U} \times 2 +1 & l=\texttt{C}
\end{cases} \\
\sigma_{\texttt{SM}}(\mathbf{b}) &= \begin{cases}
    b_1' = b_1 +1 & b_1 > U-1 \\
    b_3' = b_3 +1 & b_3 > U-1 \\
    b_1' = b_3' = \frac{b_1+b_3+1}{2} & o.w.
\end{cases} \\
\end{split}
\]

\paragraph{Experimental setup.} As discussed in the Introduction, prior methods were able to solve the classical (first) variant of the Cinderella-Stepmother game for different fixed values of the maximum bucket capacity $U < 2$. However, no prior method could compute a winning strategy for $\REACH$ player which works for any bucket capacity $2 - \epsilon$, where $\epsilon > 0$ is a symbolic variable which is provided as strategy input. 

To evaluate our method and the baselines, for both game variants we first consider different fixed values of bucket capacities $U \in \{1.5,1.7,1.9,2-10^{-10}\}$. We then consider the bucket capacity $2 - \epsilon$, where $\epsilon > 0$ is a symbolic variable. To implement this, we declare $\epsilon$ as a variable, add the constraints $\epsilon > 0$ and $U = 2 - \epsilon$ to the set of allowed valuations of variables $X$, and add $\epsilon' = \epsilon$ to each transition. This is supported by our tool and \texttt{GenSys}, but not by \texttt{rpgsolve}.\footnote{We also tried providing \texttt{rpgsolve} with a variant in which Cinderella chooses the value of $\epsilon > 0$ in the first step, but this is also not supported as the syntax of \texttt{rpgsolve} as it requires \texttt{SAFE} player to only have finitely many available actions at each step.}

\paragraph{Experimental results.} The results are shown in Table~\ref{tab:experiments-cinderella}.  
Our method is able to solve the classical (first) variant of the Cinderella-Stepmother game for all considered bucket capacities in less than 1s, including the most general case $2 - \epsilon$ where $\epsilon > 0$ is a symbolic variable. In contrast, both $\texttt{GenSys}$ and $\texttt{rpgsolve}$ fail on the largest fixed bucked capacity $2 - 10^{-10}$ and on the most general case $2 - \epsilon$. The results of our method are a significant achievement, given that automated solving of the Cinderella-Stepmother game
for the general case of bucket capacity $2 - \epsilon$ was open and this is the first time it is solved by an automated method.

Our method is also able to solve the polynomial (second) variant of the Cinderella-Stepmother game for all considered fixed bucket capacities in less than 1s, but it fails on the most general case $2 - \epsilon$. However, it significantly outperforms the baselines on this more challenging benchmark as both baselines fail to compute a winning strategy for any considered bucket capacity. This shows that our method advances the state of the art in infinite-state reachability game solving on benchmarks that were beyond the reach of existing methods.

\paragraph{Second Benchmark.} As a second benchmark, we consider a game that models a faulty robot. Intuitively, the robot is supposed to mix two chemicals to obtain a mixture that has a specific density. However, the robot is faulty and might leak some extra amounts of chemicals in each step. This can be formally modeled by a pair $(a,b)$ where $a$ is the amount of the first substance, and $b$ is the amount of the second one. In each turn, as the $\REACH$ player, the robot can add at most 1 liter from the combination of the chemicals to the mixture. On the other hand, the faultiness of the robot is modeled by an adversary that can increase each of $a$ and $b$ by $0.2$ in every other step. The objective of the robot is to make at least 10 liters of the mixture that contains at least 10\% of the first chemical, i.e. the goal is to reach $(a,b)$ such that $a>9b$ and $a+b>10$. Given that both players have infinite action-spaces, the game is not supported by \texttt{rpgoslve}. 

\paragraph{Results.} We ran our tool together with the baselines on this game. While our method solves the game in 44 seconds, \texttt{GenSys} does not terminate in 5 minutes. Once more, this example shows the efficiency of our approach for solving reachability games that are difficult to solve for other approaches. 


\begin{table}[t]
	\centering
        \footnotesize{
	\begin{tabular}{|c|c|c|c|c|}
		\hline
		                                                    & $U$          & Ours        & \texttt{GenSys} & \texttt{rpgsolve} \\
		\hline
		\multirow{4}{*}{\rotatebox{90}{\makebox[0pt]{Classical}}}             & 1.5          & 0.4        & \textbf{0.3}            & 0.4              \\
		                                                    & 1.7          & 0.4        & 0.5            & \textbf{0.3}              \\
		                                                    & 1.9          & \textbf{0.3}& 2.8            & 0.7              \\
 & $2-10^{-10}$& \textbf{0.4}& \texttt{TO}&\texttt{TO}\\
		                                                    & $2-\epsilon$ & \textbf{0.9}& \texttt{TO}      & \texttt{NS}    \\
		\hline
		\multirow{4}{*}{\rotatebox{90}{\makebox[0pt]{\small Non-linear}}} & 1.5          & \textbf{0.4}        & \texttt{TO}     & \texttt{TO}\\
		                                                    & 1.7          & \textbf{0.4}        & \texttt{TO}     & \texttt{TO}\\
		                                                    & 1.9          & \textbf{0.5}        & \texttt{TO}     & \texttt{TO}\\
 & $2-10^{-10}$& \textbf{0.6}& \texttt{TO}& \texttt{TO}\\
		                                                    & $2-\epsilon$ & \texttt{TO} & \texttt{TO}     & \texttt{NS}       \\
		\hline
	\end{tabular}
        }
	\caption{Our experimental results on two variants of the Cinderella-Stepmother game. Each cell displays the time (in seconds) it takes for a tool to solve the respective benchmark. \texttt{TO} and \texttt{NS} stand for "timeout" and "not supported".}
	\label{tab:experiments-cinderella}
\end{table}

\section{Conclusion} \label{sec:conclu}

We introduced ranking certificates for reachability games, a new sound and complete proof rule for proving existence of winning strategies of $\REACH$ player in infinite-state reachability games. We then presented an automated method for solving infinite-state polynomial reachability games using ranking certificates and template-based synthesis. Our automated method is the first to provide: (i) soundness and semi-completeness guarantees, (ii) sub-exponential runtime, and (iii) support for polynomial reachability games. Our prototype outperforms state-of-the-art tools and for the first time solves the classical Cinderella-Stepmother game with general bucket capacity of $2-\epsilon$, for any symbolic $\epsilon > 0$. Looking forward, extending our method to broader objectives (e.g., $\omega$-regular and quantitative objectives) and to non-polynomial games are interesting directions of future work.


\pagebreak
\section*{Acknowledgements}
This research was supported by Vienna Science and Technology Fund
(WWTF), State of Lower Austria [Grant ID 10.47379/ICT25017], ERC CoG 863818 (ForM-SMArt), and Austrian Science Fund (FWF) 10.55776/COE12.
\bibliographystyle{named}
\bibliography{bibliography}

\pagebreak
\onecolumn
\twocolumn
\appendix

\section{Other Related Works}
Another line of research focuses on abstraction refinement techniques~\cite{DBLP:conf/lics/BallK06,DBLP:journals/iandc/GrumbergLLS07,DBLP:conf/icalp/HenzingerJM03}. A common approach is counter example guided synthesis where based on counter examples, current abstraction gets refined. These algorithms are not complete for systems with infinite number of bistimulation classes and their refinement may not terminate if the strategy needs unbounded loops.
The dual problem of safety has been also considered~\cite{DBLP:conf/aplas/MarkgrafHLNN20,DBLP:conf/tacas/KatisFGGBGW18}. None of these methods provide general completeness, so, are not applicable to solve a reachability games.
Finally, There are works that consider just one type of player for reachability and LTL synthesis~\cite{asadi,DBLP:conf/fm/ChatterjeeGGKZ24}, i.e., all the actions are picked by either $\SAFE$ or $\REACH$.

Specifically, we discuss both positive and negative aspects of the comparison of our method against~\cite{FarzanK18} and~\cite{DBLP:journals/pacmpl/HeimD24}. On the positive side, our method is the first one which (i) is applicable to non-linear games (unlike~\cite{FarzanK18}), and (ii) provides completeness guarantees (unlike~\cite{DBLP:journals/pacmpl/HeimD24}). Moreover, our experiments show that our method outperforms~\cite{DBLP:journals/pacmpl/HeimD24} on both linear and non-linear settings.

On the negative side, the methods of~\cite{DBLP:journals/pacmpl/HeimD24} and~\cite{FarzanK18} can in principle produce {\em piecewise strategies} for the REACH player, whereas our template-based synthesis method produces strategies represented via a single polynomial expression. While piecewise linear functions over any compact domain can be approximated via polynomial functions up to arbitrary precision due to the Stone-Weierstrasss theorem, in practice this may require using large polynomial degrees. So we imagine there would be examples where this could be an issue for our method, although we did not observe this in our benchmarks.

\section{Example: Compact Action Space Without a Ranking Certificate}
\label{app:compact-example}

The completeness result of Theorem~\ref{thm:completeness} requires that $\SAFE$ has only finitely many choices at each step. The following example illustrates that this assumption cannot be relaxed to a compact (but infinite) action space: even when $\SAFE$'s action space is compact, $\REACH$ may have a winning strategy and yet no ranking certificate exists.

\begin{example} \label{ex:compact-example}
    Consider a game with two real-valued variables $(x, y)$ initialized to $(0, 1)$.
    At each step, $\SAFE$ chooses a value for $x$ from the closed interval $[0,1]$.
    \begin{itemize}
        \item If $\SAFE$ chooses $x = 0$, the game ends immediately and $\REACH$ wins.
        \item Otherwise, the update $y' = y - x$ is applied repeatedly as long as $y \geq 0$, and $\REACH$ wins when $y < 0$ is first reached.
    \end{itemize}
    No matter which value $x \in [0,1]$ the $\SAFE$ player picks (as long as $x > 0$), $y$ decreases by $x$ at every step, so $\REACH$ eventually wins.
    Hence, $\REACH$ has a winning strategy.
    However, for any fixed $\varepsilon > 0$, $\SAFE$ can choose $x = \varepsilon$, forcing the game to last for $\lceil 1/\varepsilon \rceil$ steps before $\REACH$ wins.
    Since $\SAFE$ can make $\varepsilon$ arbitrarily small, there is no finite upper bound on the number of steps required for $\REACH$ to win.
    Therefore, no ranking certificate---which by definition provides such a finite bound---can exist for this game, even though $\REACH$ has a winning strategy and $\SAFE$'s action space $[0,1]$ is compact.
\end{example}

\begin{figure}[h]
    \centering
    \includegraphics[width=0.75\columnwidth]{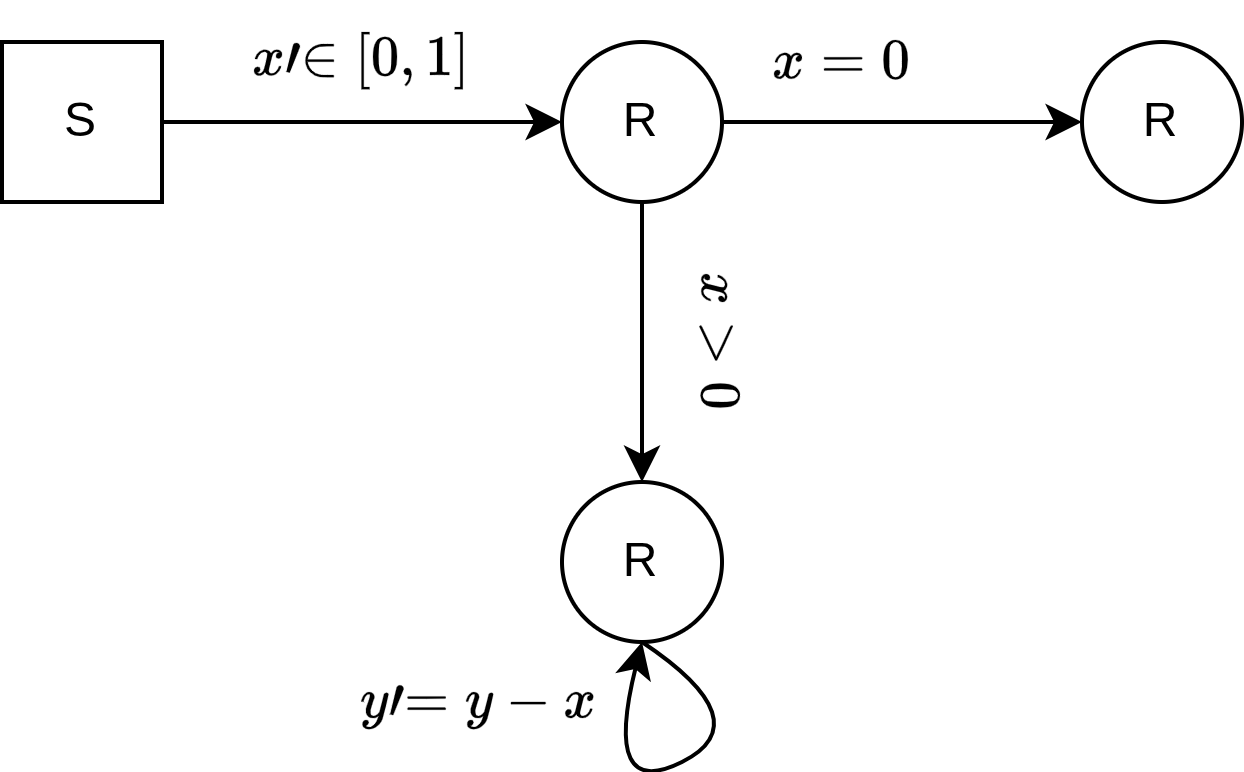}
    \caption{Illustration of the game from Example~\ref{ex:compact-example}. $\SAFE$ picks $x \in [0,1]$; choosing $x=0$ is an immediate win for $\REACH$, while any $x>0$ leads to the update $y' = y - x$, which eventually drives $y$ below $0$ (winning for $\REACH$) after an unbounded number of steps.}
    \label{fig:compact-example}
\end{figure}

\section{Proof of Theorem \ref{thm:soundness}} \label{app:proof-rule-sound}
\begin{theorem*}[Soundness]
    Suppose that there exists a ranking certificate $f$ in a reachability game $\game$ with a set of target states $\target \subseteq \states$. Then, there exists a memoryless winning strategy for $\REACH$ player.

    An example of such a strategy is a strategy $\rstrat^f$ that maps each state $s \in \rstates$ to a ranking successor of $s$ if $f(s) \geq 0$, and to an arbitrary successor of $s$ if $f(s) < 0$.
\end{theorem*}

\begin{proof}
    Let $\rstrat^f$ be a memoryless strategy described in the theorem statement. To prove the theorem, it suffices to prove that $\rstrat^f$ is a winning strategy for $\REACH$. In particular, we need to show that for every strategy $\sigma$ of $\SAFE$, the play $\rho(\rstrat^f,\sigma)$ is winning.

    To prove this, we fix a strategy $\sigma$ of $\SAFE$ and prove that the play $\rho(\rstrat^f,\sigma)$ eventually visits a state in the target set $\target$. By condition (C1) in Definition~\ref{def:reach-rank}, the value of $f$ at the initial state of the play is non-negative. Then, condition (C2) ensures that no matter what the player $\SAFE$ does, if the play visits a non-target state that belongs to $\SAFE$ at which the value of $f$ is non-negative, then the value of $f$ will decrease by at least $1$ in the next state of the play while remaining non-negative. Finally, condition (C3) ensures that if the play visits a non-target state that belongs to $\REACH$ at which the value of $f$ is non-negative, then the value of $f$ at the ranking successor state of the play chosen by the strategy $\rstrat^f$ will decrease by at least $1$ in the next state of the play while remaining non-negative. Hence, the value of $f$ along the states of the play has to remain non-negative and to decrease by at least $1$ at every step, until a target set is visited. Thus, the target set must be visited in at most $\lceil f(s_\init) \rceil$ steps. This concludes our proof.
\end{proof}

\section{Proof of Theorem \ref{thm:completeness}} \label{app:proof-rule-thm}
\begin{theorem*}[Completeness for finite $\SAFE$ choices]
    Let $\game$ be a reachability game with a set of target states $\target \subseteq \states$. Suppose that $\SAFE$ has only finitely many choices at each step, i.e. for every state $s \in \sstates$ the set of successor states $\succ(s)$ is finite. If $\REACH$ has a winning strategy, then there exists a ranking certificate for this reachability game.
\end{theorem*}

\begin{proof} Let $\rstrat$ be a memoryless winning strategy of $\REACH$.
    A state $s \in \states$ is called $\rstrat$-\connected if there is a run of $\game$ that conforms to $\rstrat$ and contains $s$. Let $\states_{\rstrat}$ be the set of all $\rstrat$-\connected states of $\game$. We define the function $f\colon \states \to \R$ as follows:
    \[
        f(s) = \begin{cases}
            -1                                    & s \notin \states_{\rstrat}                              \\
            0                                     & s \in \states_{\rstrat} \cap \target                    \\
            f(\rstrat(s))+1                       & s \in \states_{\rstrat} \cap \rstates\setminus\target   \\
            \sup\limits_{s' \in \succ(s)} f(s')+1 & s \in \states_{\rstrat} \cap \sstates \setminus \target
        \end{cases}
    \]
    We show that the value of $f$ is always finite and that it is a $\REACH$-ranking function for $\game$. Suppose, for the sake of contradiction, that $f(s) = \infty$ for some state $s$. It immediately follows that $f(s_\init) = \infty$. Consider the run $\rho = (s_0, s_1, \dots)$ as follows: (i) $s_0 = s_\init$, (ii) if $s_i \in \rstates$, then $s_{i+1} = \rstrat(s_i)$, and (iii) if $s_i \in \sstates$, then $s_{i+1}$ is one of the successors of $s_i$ whose $f$-value is not well-defined. Such $s_{i+1}$ always exists since $f(s_i)$ is the supremum of $f$-values of its successors plus one and each $\succ(s_i)$ is assumed to be finite for each $s_i \in \sstates$. By construction, for each $i$ it holds that $f(s_i) = \infty$, hence $s_i \notin \target$ which means $\rho$ is a play of $\game$ that conforms to $\rstrat$ and is not winning for $\REACH$. This is a contradiction with the definition of $\REACH$-winning strategies.

    Next we show that $f$ is a $\REACH$-ranking function by Def.~\ref{def:reach-rank}. Firstly, note that for each $s \in \states_{\rstrat}$ it holds by construction that $f(s) \geq 0$. We show that the three conditions in Def.~\ref{def:reach-rank} holds as follows:
    \begin{enumerate}[1.]
        \item $f(s_\init) \geq 0$ since $s_\init \in \states_{\rstrat}$.
        \item For each $s \in \states_{\rstrat} \cap \rstates \setminus \target$ it holds by construction that $f(s) - 1 = f(\rstrat(s)) \geq 0$.
        \item For each $s \in \states_{\rstrat} \cap \sstates \setminus \target$ it holds by construction that $f(s) -1 = \sup\limits_{s'' \in \succ(s)} f(s'') \geq f(s') \geq 0$, for all $s' \in \succ(s)$.
    \end{enumerate}
    It follows that $f$ is a $\REACH$-ranking function for $\game$.
\end{proof}

\section{Proof of Theorem~\ref{thm:algo}}
\begin{theorem*}
    Consider a polynomial reachability game $\game = (\vars, \valspace, \valinit, \rlabels, \slabels, \labelinit, \trans)$ with target states $\target$:
    \begin{compactenum}[1.]
        \item {\em (Soundness)} If the algorithm returns a solution $(f,\strategy)$, then $f$ is a valid ranking certificate for this reachability game and $\strategy$ is a memoryless winning strategy of $\REACH$.
        \item {\em (Semi-completeness)} Whenever the set of possible valuations $\valspace$ is compact, if there exists a polynomial ranking certificate $f$ and a polynomial memoryless winning strategy $\strategy$ for this reachability game then, for a sufficiently high value of the maximal polynomial degree parameter $D$, the algorithm is guaranteed to compute them. 
        \item {\em (Complexity)} The runtime of the algorithm is sub-exponential in the size of the input game $\game$ and $\target$, for each fixed value of the maximal polynomial degree $D$.
    \end{compactenum}
\end{theorem*}
\begin{proof}
    {\em (Soundness)} Suppose the algorithm returns $(f,\strategy)$. Then, given the soundness of the translation done in Step~3, the pair $(f,\strategy)$ satisfies the constraints collected in Step~2. This means that conditions (C1), (C2) and (C3) are all satisfied by $(f,\strategy)$ which makes $f$ a $\REACH$-ranking function and $\strategy$ is a strategy for $\REACH$ that always choses a ranking successor, hence it is a winning strategy by Theorem~\ref{thm:soundness}.

    {\em (Semi-completeness)} The system of constraints collected up to Step~2 of the algorithm is satisfiable if and only if a pair $(f,\strategy)$ exists that satisfies the conditions in Definition~\ref{def:reach-rank}. Given that $\valspace$ is a compact set, it follows that the translation via Farkas' and Putinar's theorem in Step~3 is sound and complete for a fixed polynomial degree $D$~\cite{asadi}. Finally, the constraint solving of Step 4 is sound and complete, implying the semi-completeness claim.

        {\em (Complexity)} For any fixed degree $D$ of the template polynomials, our algorithm reduces the problem of synthesizing polynomial $(f,\strategy)$ pairs to Quadratically-constrained Quadratic Programming (QP) in polynomial time. It is well-known that QP is solvable in sub-exponential time \cite{grigor1988solving}. Hence, our algorithm runs in sub-exponential time as well.
\end{proof}

\section{Faulty Robot Benchmark: Full Description and Results}
\label{app:faulty-robot}

As a second benchmark, we consider a reachability game that models a faulty chemical-mixing robot. The robot is supposed to mix two chemicals to produce a mixture with a specific concentration. However, due to hardware faults (modeled as an adversary), the robot may suffer leaks that add extra amounts of chemicals at each step.

\paragraph{Formal game definition.}
The game is defined as the tuple $\game = (\vars, \valspace, \valinit, \rlabels, \slabels, \labelinit, \trans)$ where:
\begin{compactitem}
    \item $\vars = \{a, b\}$, where $a \geq 0$ denotes the accumulated amount of the first chemical and $b \geq 0$ denotes the accumulated amount of the second chemical in the mixture.
    \item $\valspace = \{(a,b) \mid a \geq 0,\, b \geq 0\}$.
    \item $\valinit = (0, 0)$, i.e., the mixture is initially empty.
    \item $\rlabels = \{\texttt{Robot}\}$ and $\slabels = \{\texttt{Fault}\}$ are the sets of $\REACH$ and $\SAFE$ labels, respectively.
    \item $\labelinit = \texttt{Robot}$, i.e., the robot acts first.
    \item $\trans$ consists of two transitions:
    \begin{compactenum}[1.]
        \item $\tau_R = (\texttt{Robot},\, \texttt{Fault},\, \valspace,\, \update_R)$, where the update relation is
        \begin{multline*}
            \update_R = \{((a,b),(a',b')) \mid \delta_a, \delta_b \geq 0,\; \delta_a + \delta_b \leq 1,\\
            a' = a + \delta_a,\; b' = b + \delta_b\}.
        \end{multline*}
        The robot ($\REACH$) chooses any non-negative split $(\delta_a, \delta_b)$ with $\delta_a + \delta_b \leq 1$ to add to the mixture.
        \item $\tau_F = (\texttt{Fault},\, \texttt{Robot},\, \valspace,\, \update_F)$, where the update relation is
        \begin{multline*}
            \update_F = \{((a,b),(a',b')) \mid 0 \leq \varepsilon_a \leq 0.2,\; 0 \leq \varepsilon_b \leq 0.2,\\
            a' = a + \varepsilon_a,\; b' = b + \varepsilon_b\}.
        \end{multline*}
        The adversary ($\SAFE$) models a fault that independently leaks up to $0.2$ units of each chemical into the mixture.
    \end{compactenum}
\end{compactitem}

\paragraph{Objective.}
The target set is $\target = \{(\texttt{Fault},\, (a,b)) \mid a > 9b \text{ and } a + b > 10\}$.
That is, $\REACH$ wins as soon as, after the robot's action, the mixture exceeds 10 liters in total volume and the first chemical accounts for more than $90\%$ of the mixture (i.e., $a/(a+b) > 9/10$).

\end{document}